\documentclass{llncs}
\usepackage{amsmath,amssymb,amsfonts}
\usepackage{mathtools}
\usepackage[usenames,dvipsnames]{color}
\usepackage{xcolor}
\usepackage{xspace}
\usepackage{microtype}
\usepackage{todonotes}
\usepackage{colonequals}
 \usepackage{booktabs}
\usepackage{cite}

\usepackage{multirow}
\usepackage{multicol}

\usepackage{paralist}

\usepackage{subfigure}
\usepackage{url}
\usepackage{appendix}
\usepackage{graphicx}
\usepackage{algorithm}
\usepackage{listings}
\usepackage{nicefrac}
\usepackage{tikz} 
\usepackage{pgfplots}
\usepackage{mdframed}

\usetikzlibrary{arrows,decorations.pathmorphing,positioning,fit,trees,shapes,shadows,automata,calc} 

\tikzset{outline/.style args={#1}{%
  draw=#1,thick,fill=#1!50},initial text={}}

\tikzset{
  dot hidden/.style={},
  line hidden/.style={},
  dice hidden/.style={},
  dot color/.style={dot hidden/.append style={color=#1}},
  dot color/.default=black,
  line color/.style={line hidden/.append style={color=#1}},
  line color/.default=black,
  dice color/.style={dice hidden/.append style={color=#1,fill}},
  dice color/.default=white
}\def\dotsize{0.1}
\newcommand{\drawdie}[2][]{%
\begin{tikzpicture}[x=1em,y=1em,#1]
  \draw 	[thick, rounded corners=0.5,line hidden,dice hidden] (0,0) rectangle (1,1);
  \ifodd#2
    \fill[dot hidden] (0.5,0.5) circle (\dotsize);
  \fi
  \ifnum#2>1
  \fill[dot hidden] (0.25,0.25) circle (\dotsize);
  \fill[dot hidden] (0.75,0.75) circle (\dotsize);
  \ifnum#2>3
    \fill[dot hidden] (0.25,0.75) circle (\dotsize);
    \fill[dot hidden] (0.75,0.25) circle (\dotsize);
    \ifnum#2>5
      \fill[dot hidden] (0.75,0.5) circle (\dotsize);
      \fill[dot hidden] (0.25,0.5) circle (\dotsize);
    \fi
  \fi
\fi
\end{tikzpicture}
}

\usepackage{etoolbox}

\newtoggle{TR}
\toggletrue{TR}
\title{Synthesis in pMDPs: A Tale of 1001 Parameters}
\author{Murat Cubuktepe\inst{1}, Nils Jansen\inst{2}, Sebastian Junges\inst{3},\\ Joost-Pieter Katoen\inst{3}, Ufuk Topcu\inst{1}}
\institute{The University of Texas at Austin, Austin, TX, USA\thanks{Supported by the grants ONR
N000141613165, NASA NNX17AD04G and AFRL FA8650-15-C-2546} \and Radboud University, Nijmegen, The Netherlands  \and RWTH Aachen University, Aachen, Germany\thanks{Supported by the CDZ project CAP (GZ 1023), and the DFG RTG 2236 ``UnRAVeL''.}
}

\newcommand{\ie}{i.e.\@\xspace}

\newcommand{\prophesy}{\textrm{PROPhESY}\xspace}
\newcommand{\prism}{\textrm{PRISM}\xspace}
\newcommand{\storm}{\textrm{Storm}\xspace}
\newcommand{\param}{\textrm{PARAM}\xspace}
\newcommand{\tool}[1]{\textrm{#1}\xspace}

\newcommand{\TO}{TO}
\newcommand{\MO}{MO}
\newcommand{\highlight}[1]{\textcolor{blue!50!black}{\textbf{#1}}}

\newcommand{\epsgraph}{\varepsilon_\text{graph}}

\newcommand{\dtmc}{\mathcal{D}}

\newcommand{\p}{\ensuremath{\mathbb{P}}}
\newcommand{\pr}{\ensuremath{\mathrm{Pr}}}

\newcommand{\reachProp}[2]{\ensuremath{\p_{\leq #1}(\finally #2)}}
\newcommand{\reachProplT}{\ensuremath{\reachProp{\lambda}{T}}}

\newcommand{\reachPropSymbol}{\varphi_r}
\newcommand{\ereachPropSymbol}{\varphi_c}

\newcommand{\expRewProp}[2]{\ensuremath{\EV_{\leq #1}(\finally #2)}}
\newcommand{\expRewPropkT}{\ensuremath{\expRewProp{\kappa}{T}}}
\newcommand{\rewFunction}{\ensuremath{{c}}}

\newcommand{\finally}{\lozenge}

\newcommand{\R}{\mathbb{R}}

\newcommand{\Ireal}{[0,\, 1]\subseteq\mathbb{R}}  
\newcommand{\EV}{\ensuremath{\mathbb{E}}}

\newcommand{\Distr}{\mathit{Distr}}

\newcommand{\distDom}{X}

\newcommand{\distFunc}{\mu}
\newcommand{\distDomElem}{x}

\newcommand{\Paramvar}{\ensuremath{{V}}\xspace}        

\newcommand{\State}{\ensuremath{\mathbb{S}}\xspace}      

\newcommand{\sinit}{s_{\mathit{I}}} 
\newcommand{\mdp}{\mathcal{M}}

\newcommand{\pMdpInit}[1][]{\ensuremath{\mdp{#1}=(S{#1},\,\sinit{#1},\Act,\Paramvar,\probmdp{#1})}}
\newcommand{\probmdp}{\mathcal{P}}

\newcommand{\sched}{\ensuremath{\sigma}}
\newcommand{\Sched}{\ensuremath{\mathit{Str}}}

\newcommand{\Act}{\ensuremath{\mathit{Act}}}
\newcommand{\ActS}{\ensuremath{\mathit{A}}}

\newcommand{\act}{\ensuremath{\alpha}}
\newcommand{\pmdp}{\ensuremath{\mathcal{P}}}

\DeclareMathAlphabet{\mathpzc}{OT1}{pzc}{m}{it}
\def\presuper#1#2%
  {\mathop{}%
   \mathopen{\vphantom{#2}}^{#1}%
   \kern-\scriptspace%
   #2}

\begin{document}

\maketitle

\begin{abstract}
This paper considers parametric Markov decision processes (pMDPs) whose transitions are equipped with affine functions over a finite set of parameters.
The synthesis problem is to find a parameter valuation such that the instantiated pMDP satisfies a (temporal logic) specification under all strategies. 
We show that this problem can be formulated as a quadratically-constrained quadratic program (QCQP) and is non-convex in general.
To deal with the NP-hardness of such problems, we exploit a convex-concave procedure (CCP) to iteratively obtain local optima.
An appropriate interplay between CCP solvers and probabilistic model checkers creates a procedure --- realized in the tool PROPheSY --- that solves the synthesis problem for models with thousands of parameters. 
\end{abstract}

\section{Introduction}
\label{sec:introduction}

\emph{The parameter synthesis problem.}
Probabilistic model checking concerns the automatic verification of models such as Markov decision processes (MDPs).
Unremitting improvements in algorithms and efficient tool implementations~\cite{KNP11,DBLP:conf/cav/DehnertJK017,iscasmc} have opened up a wide variety of applications, most notably in dependability, security, and performance analysis as well as systems biology.
However, at early development stages, certain system quantities such as fault or reaction rates are often not fully known. 
This lack of information gives rise to parametric models where transitions are functions over real-valued parameters~\cite{Daw04,lanotte,param_sttt}, forming 
symbolic descriptions of (uncountable) families of concrete MDPs. 
The \emph{parameter synthesis problem} is: Given a finite-state parametric MDP, find a parameter instantiation such that the induced concrete model satisfies a given specification.
An inherent problem is \emph{model repair}, where probabilities are changed (``repaired'') with respect to parameters such that a model satisfies a specification~\cite{bartocci2011model}.
Concrete applications include adaptive software systems~\cite{calinescu-et-al-cacm-2012}, sensitivity analysis~\cite{su-et-al-icse-2016-qosevaluation}, and optimizing randomized distributed algorithms~\cite{DBLP:conf/srds/AflakiVBKS17}.


\emph{State-of-the-art.}
First approaches to parameter synthesis compute a rational function over the parameters to symbolically express reachability probabilities~\cite{Daw04,param_sttt,DBLP:journals/corr/abs-1804-01872}.
Equivalently,\cite{DBLP:journals/tse/FilieriTG16,DBLP:journals/corr/abs-1709-02093} employ Gaussian elimination for matrices over the field of rational functions. 
Solving the (potentially very large, high-degree) functions is naturally a SAT-modulo theories (SMT) problem over non-linear arithmetic, or
 a nonlinear program (NLP)~\cite{bartocci2011model,DBLP:conf/tacas/Cubuktepe0JKPPT17}.
However, solving such SMT problems is \emph{exponential in the degree of functions and the number of variables}~\cite{DBLP:journals/corr/abs-1709-02093}, and solving NLPs is \emph{NP-hard in general}~\cite{Linderoth2005}.
%
%
%
%
Specific approaches to model repair rely on NLP~\cite{bartocci2011model} or particle-swarm optimization (PSO)~\cite{chen2013model}. 

Finally, parameter synthesis is equivalent to computing finite-memory strategies for partially observable MDPs (POMDPs)~\cite{DBLP:journals/corr/abs-1710-10294}.
Such strategies may be obtained, for instance, by employing sequential quadratic programming (SQP)~\cite{amato2006solving}.
Exploiting this approach is not practical, though, because SQP for our setting already requires a (feasible) solution satisfying the given specification.
%
Overall, efficient implementations 
 in tools like \tool{PARAM}~\cite{param_sttt}, \tool{PRISM}~\cite{KNP11}, and \prophesy~\cite{dehnert-et-al-cav-2015} can handle thousands of states but only a \emph{handful of parameters}.



\emph{Our approach.}
We overcome the restriction to few parameters by employing convex optimization~\cite{boyd_convex_optimization}.
This direction is not new; \cite{DBLP:conf/tacas/Cubuktepe0JKPPT17} describes a convexification of the NLP into a geometric program~\cite{boyd2007tutorial}, which can still only handle up to about ten parameters.
We take a different approach. 
First, we transform the NLP formulation~\cite{bartocci2011model,DBLP:conf/tacas/Cubuktepe0JKPPT17} into a \emph{quadratically-constrained quadratic program} (QCQP). 
As such an optimization problem is nonconvex in general, we cannot resort to polynomial-time algorithms for convex QCQPs~\cite{alizadeh2003second}.
Instead, to solve our NP-hard problem, 
we massage the QCQP formulation into a \emph{difference-of-convex} (DC) problem.
The convex-concave procedure (CCP)~\cite{lipp2016variations} yields local optima of a DC problem by a convexification towards a convex quadratic program, which is 
amenable for state-of-the-art solvers such as \tool{Gurobi}~\cite{gurobi}.

Yet, blackbox CCP solvers~\cite{park2017general,shen2016disciplined} suffer from severe numerical issues and can only solve very small problems.
We integrate the procedure with a probabilistic model checker, creating a method that --- realized in the open-source tool PROPheSY~\cite{dehnert-et-al-cav-2015} ---
yields (a) an improvement of multiple orders of magnitude compared to just using CCP as a black box and (b) ensures the correctness of the solution.
In particular, we use probabilistic model checking to:
\begin{compactitem}
\item rule out feasible solutions that may be spurious due to numerical errors,
\item check if intermediate solutions are already feasible for earlier termination,
\item compute concrete probabilities from intermediate parameter instantiations to avoid potential numerical instabilities.
\end{compactitem}
An extensive empirical evaluation on a large range of benchmarks shows that our approach can solve the parameter synthesis problem for models with large state spaces and up to thousands of parameters, and is superior to all existing parameter synthesis tools~\cite{PARAM10,KNP11,dehnert-et-al-cav-2015}, geometric programming~\cite{DBLP:conf/tacas/Cubuktepe0JKPPT17}, and an efficient re-implementation of PSO~\cite{chen2013model} that we create to deliver a better comparison. 
Contrary to the geometric programming approach in \cite{DBLP:conf/tacas/Cubuktepe0JKPPT17}, we compute solutions that hold for all possible (adversarial) schedulers for parametric MDPs.
Traditionally, model checking delivers results for such adversarial schedulers~\cite{BK08}, which are for instance useful when the nondeterminism is not controllable and induced by the environment, which is the case in the example below.

\emph{An illustrative example.}
Consider the Carrier Sense Multiple Access/Collision Detection (CSMA/CD) protocol in Ethernet networks, 
which was subject to probabilistic model checking~\cite{DBLP:journals/entcs/DuflotFHLMMPP05}.
When two stations simultaneously attempt sending a packet (giving rise to a collision), a so-called randomized exponential back-off mechanism is used to avoid the collision.
Until the $k$-th attempt, a delay out of $2^k$ possibilities is randomly drawn from a uniform distribution.
An interesting question is if a uniform distribution is optimal, where optimality refers to the minimal expected time until all packets have been sent.
A bias for small delays seems beneficial, but raises the collision probability.
Using our novel techniques, within a minute we synthesize a different distribution, which induces less expected time compared to the uniform distribution. The used model has about $10^5$ states and 26 parameters.
We are not aware of any other parameter-synthesis approach being able to generate such a result within reasonable time.

\section{Preliminaries}
\label{sec:preliminaries}

A \emph{probability distribution} over a finite or countably infinite set $\distDom$ is a function $\distFunc\colon\distDom\rightarrow\Ireal$ with $\sum_{\distDomElem\in\distDom}\distFunc(\distDomElem)=1$. 
The set of all distributions on $\distDom$ is denoted by $\Distr(\distDom)$.
Let $V=\{x_1,\ldots,x_n\}$ be a finite set of \emph{variables} over the real numbers $\R$. The set of multivariate polynomials over $V$ is $\mathbb{Q}[V]$. An \emph{instantiation} for $V$ is a function $u\colon V \rightarrow \R$.

A function $f\colon\R^n\to\R$ is \emph{affine} if 
$f(\vec{x})=a^T\vec{x}+b$ with $a \in \R^n$ and $b \in \R$, and $f\colon\R^n\to\R$ is \emph{quadratic} if $f(\vec{x})=\vec{x}^TP\vec{x}+a^T\vec{x}+b$ with $a,b$ as before and $P \in \R^{n \times n}$.
A symmetric matrix $P \in \R^{n \times n}$ is \emph{positive semidefinite} (PSD) if $\vec{x}^TP\vec{x}\geq 0 \enskip \forall \vec{x} \in \R^n$, or equivalently, if all eigenvalues of $P$ are nonnegative. 
%
\begin{definition}[(Affine) pMDP]\label{def:pmdp}
A \emph{parametric Markov decision process (pMDP)} is a tuple $\pMdpInit$ with a finite set $S$ of \emph{states}, an \emph{initial state} $\sinit \in S$, a finite set $\Act$ of \emph{actions}, a finite set $\Paramvar$ of real-valued variables \emph{(parameters)} and a \emph{transition function} $\probmdp \colon S \times \Act \times S \rightarrow \mathbb{Q}[V]$.
A pMDP is \emph{affine} if $\probmdp(s,\act,s')$ is affine for every $s,s'\in S$ and $\act\in\Act$.
\end{definition}
%
%
%
For $s \in S$,  $\ActS(s) = \{\act \in \Act \mid \exists s'\in S.\,\probmdp(s,\,\act,\,s') \neq 0\}$ is the set of \emph{enabled} actions at $s$.
Without loss of generality, we require $\ActS(s) \neq \emptyset$ for $s\in S$.
If $|\ActS(s)| = 1$ for all $s \in S$, $\mdp$ is a \emph{parametric discrete-time Markov chain (pMC)}. We denote the transition function for pMCs by $\pmdp(s,s')$.
MDPs can be equipped with a state--action \emph{cost function} $\rewFunction \colon S \times \Act \rightarrow \R_{\geq 0}$.

A pMDP $\mdp$ is a \emph{Markov decision process (MDP)} if the transition function yields \emph{well-defined} probability distributions, \ie, $\probmdp \colon S \times \Act \times S \rightarrow [0,1]$ and $\sum_{s'\in S}\probmdp(s,\act,s') = 1$ for all $s \in S$ and $\act \in \ActS(s)$. 
Applying an \emph{instantiation} $u\colon V \rightarrow \R$ to a pMDP $\mdp$ yields $\mdp[u]$ by replacing each $f\in\mathbb{Q}[V]$ in $\mdp$ by $f[u]$.
An instantiation $u$ is \emph{well-defined} for $\mdp$ if 
the resulting model $\mdp[u]$ is an MDP.

To define measures on
MDPs, nondeterministic choices are resolved by a so-called \emph{strategy} $\sched\colon S\rightarrow\Act$ with $\sched(s) \in \ActS(s)$.
The set of all strategies over $\mdp$ is $\Sched^\mdp$.
%
%
%
For the measures in this paper, memoryless deterministic strategies suffice~\cite{BK08}.
Applying a strategy to an MDP yields an \emph{induced Markov chain} where all nondeterminism is resolved.
%

For an MC $\dtmc$, the \emph{reachability specification} $\reachPropSymbol=\reachProplT$ asserts that a set $T \subseteq S$ of \emph{target states}  is reached with probability at most $\lambda\in [0,1]$.
If $\reachPropSymbol$ holds for $\dtmc$, we write $\dtmc\models\reachPropSymbol$.
Accordingly, for an \emph{expected cost specification} $\ereachPropSymbol=\expRewProp{\kappa}{G}$ it holds that $\dtmc\models\ereachPropSymbol$ if and only if the expected cost of reaching a set $G \subseteq S$ is bounded by $\kappa \in \R$.
We use standard measures and definitions as in~\cite[Ch.\ 10]{BK08}.
%
%
%
%
An MDP $\mdp$ satisfies a specification $\varphi$, written $\mdp\models\varphi$, if and only if \emph{for all} strategies $\sched\in\Sched^\mdp$ it holds that $\mdp^\sched\models\varphi$. 

\section{Formal Problem Statement}
\begin{mdframed}[backgroundcolor=gray!30]
\begin{problem}[pMDP synthesis problem]\label{prob:pmdpsyn}
Given a pMDP $\pMdpInit$, and a reachability specification $\reachPropSymbol=\reachProplT$, 
 compute a well-defined
instantiation $u\colon V \rightarrow \R$ for $\mdp$ such that $\mdp[u]\models\reachPropSymbol$.
\end{problem}
\end{mdframed}
Intuitively, we seek for an instantiation of parameters $u$ that satisfies $\reachPropSymbol$ for all possible strategies for the instantiated MDP.
We show necessary adaptions for an expected cost specification $\ereachPropSymbol=\expRewPropkT$ later.


For a given instantiation $u$, Problem~\ref{prob:pmdpsyn} boils down to verifying if $\mdp[u]\models\reachPropSymbol$. 
The standard formulation uses a linear program (LP) to minimize the probability $p_{\sinit}$ of reaching the target set $T$ from the initial state $\sinit$, while ensuring 
that this probability is realizable under any strategy~\cite[Ch.\ 10]{BK08}.
The straightforward extension of this approach to pMDPs in order to \emph{compute} a suitable instantiation $u$ yields the following nonlinear program (NLP):
		\begin{align}
			\text{minimize } &\quad p_{\sinit}\label{eq:min_mdp}\\
			\text{subject to} &\nonumber \\
			\forall s\in T.	 &\quad p_s=1\label{eq:targetprob_mdp}\\
			\forall s,s'\in S.\, \forall\act\in\Act.	 &\quad \probmdp(s,\act,s')\geq 0\label{eq:well-defined_probs_mdp}\\
						\forall s\in S.\, \forall\act\in\Act.	 &\quad \sum_{s'\in S}\probmdp(s,\act,s')=1\label{eq:well-defined_probs_mdp1}\\
&\quad  \lambda \geq p_{\sinit}\label{eq:probthreshold_mdp}\\
									\forall s\in S\setminus T.\,\forall \act\in\Act.	&\quad p_s \geq \sum_{s'\in S}	\probmdp(s,\act,s')\cdot p_{s'}\label{eq:probcomputation_mdp}.
		\end{align}
For $s \in S$, the \emph{probability variable} $p_s\in[0,1]$ represents an upper bound of the probability of reaching target set $T\subseteq S$, and the \emph{parameters} in set $V$ enter the NLP as part of the functions from $\mathbb{Q}[V]$ in the transition function $\probmdp$.
%
\begin{proposition}
The NLP in \eqref{eq:min_mdp} -- \eqref{eq:probcomputation_mdp} computes the \emph{minimal probability} of reaching $T$ under a \emph{maximizing} strategy.
\end{proposition}
The probability to reach a state in $T$ from $T$ is one~\eqref{eq:targetprob_mdp}.
The constraints~\eqref{eq:well-defined_probs_mdp} and~\eqref{eq:well-defined_probs_mdp1} ensure \emph{well-defined} transition probabilities.
Constraint~\eqref{eq:probthreshold_mdp} is optional but necessary later, and ensures that the probability of reaching $T$ is below the threshold $\lambda$.
For each non-target state $s\in S$ and any action $\act\in\Act$, the probability induced by the \emph{maximizing scheduler} is a lower bound to the probability variables $p_s$~\eqref{eq:probcomputation_mdp}.
To assign probability variables the minimal values with respect to the parameters from $V$, $p_{\sinit}$ is minimized in the objective~\eqref{eq:min_mdp}.

%
%
%
\begin{remark}[Graph-preserving instantiations]
	In the LP formulation for MDPs, states with probability $0$ to reach $T$ are determined via a preprocessing on the underlying graph, and their probability variables are set to zero, to avoid an underdetermined equation system. 	
	For the same reason, we preserve the underlying graph of the pMDP, as in~\cite{param_sttt,dehnert-et-al-cav-2015}.
	We thus exclude valuations $u$ with $f[u]=0$ for $f\in\probmdp(s,\act,s')$ for all $s,s'\in S$ and $\act\in\Act$.
We replace \eqref{eq:well-defined_probs_mdp} by \begin{align}
 	\forall s,s'\in S.\, \forall\act\in\Act.	 &\quad \probmdp(s,\act,s')\geq \epsgraph.\label{eq:well-defined_eps}
 \end{align}
 where $\epsgraph>0$ is a small constant. 
\end{remark}
\begin{example}\label{ex:nlp}
Consider the pMC in Fig.~\ref{fig:pmc_reform} with parameter set $V=\{v\}$, initial state $s_0$, and target set $T = \{s_3\}$. Let $\lambda$ be an arbitrary constant.
The NLP in~\eqref{eq:nlp-ex1} -- \eqref{eq:nlp-ex2} minimizes the probability of reaching $s_3$ from the initial state:
\begin{align}
			\text{minimize} & \quad p_{s_0} \label{eq:nlp-ex1} \\
			\text{subject to} &\nonumber\\
			&\quad p_{s_3}=1\\
			&\quad \lambda \geq p_{s_0} \geq v\cdot p_{s_1} \\
			&\quad p_{s_1} \geq (1-v)\cdot p_{s_2}\\
			&\quad  p_{s_2} \geq v\cdot p_{s_3}  \\
			&\quad 1-\epsgraph \geq v\geq \epsgraph. \label{eq:nlp-ex2}
		\end{align}
\end{example}
\begin{figure}[t]
	\centering
	\begin{tikzpicture}[scale=1, nodestyle/.style={draw,circle},baseline=(s0)]

    \node [nodestyle,initial] (s0) at (0,0) {$s_0$};
    \node [nodestyle] (s1) [on grid, right=2cm of s0] {$s_1$};
    \node [nodestyle] (s2) [on grid, right=2cm of s1] {$s_2$};
    \node [nodestyle ,accepting] (s3) [on grid, right=2cm of s2] {$s_3$};
    \node [nodestyle, gray] (s4) [on grid, below=1cm of s1] {$s_4$};

    \draw[->] (s0) -- node [auto] {\scriptsize$v$} (s1);
    \draw[->,gray] (s0) -- node [below,pos=0.3, yshift=-0.1cm] {\scriptsize$1-v$} (s4);

    \draw[->] (s1) -- node [auto] {\scriptsize$1-v$} (s2);
    \draw[->,gray] (s1) -- node [auto] {\scriptsize$v$} (s4);
    
    \draw[->] (s2) -- node [auto] {\scriptsize$v$} (s3);
    \draw[->,gray] (s2) -- node [below,pos=0.3, yshift=-0.1cm] {\scriptsize$1-v$} (s4);

    \draw(s3) edge[loop right] node [right] {\scriptsize$1$} (s3);
    \draw(s4) [gray] edge[loop right] node [right] {\scriptsize$1$} (s3);
    
\end{tikzpicture}	
	\caption{A pMC with parameter $v$.}
	\label{fig:pmc_reform}
\end{figure}
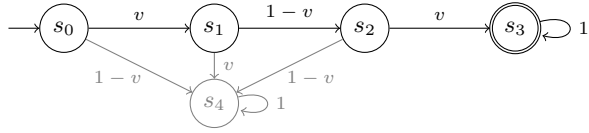

\paragraph{Expected cost specifications.}
The NLP in \eqref{eq:min_mdp} -- \eqref{eq:well-defined_eps} considers reachability probabilities.
If we have instead an expected cost specification $\ereachPropSymbol=\expRewProp{\kappa}{G}$, we replace \eqref{eq:targetprob_mdp}, \eqref{eq:probthreshold_mdp}, and \eqref{eq:probcomputation_mdp} in the NLP by the following constraints:
\begin{align}
	\forall s\in G.	 &\quad p_s=0,\label{eq:targetrew}\\
	\forall s\in S\setminus G.\, \forall\act\in\ActS(s).	&\quad p_s\geq  c(s,\act) + \sum_{s'\in S}	\probmdp (s,\act,s') \cdot p_{s'}
	 \label{eq:rewcomputation}\\
	&\quad \kappa \geq p_{\sinit}\label{eq:strategyah:lambda}.
\end{align}
We have $p_s\in\R$, as these variables represent the expected cost to reach $G$. 
At $G$, the expected cost is set to zero \eqref{eq:targetrew}, and the actual expected cost for other states is a lower bound to $p_s$ \eqref{eq:rewcomputation}.
Finally, 
$p_{\sinit}$ is bounded by the threshold $\kappa$.

\section{QCQP Reformulation of the pMDP Synthesis Problem}\label{sec:qcqp}
For the remainder of this paper, we restrict pMDPs to be affine, see Def.~\ref{def:pmdp}.
For an affine pMDP $\mdp$, the functions in the resulting NLP~\eqref{eq:min_mdp} -- \eqref{eq:well-defined_eps} for pMDP synthesis from the previous section are affine in $V$.
However, the functions in~\eqref{eq:probcomputation_mdp} are \emph{quadratic}, as a result of multiplying affine functions occurring in $\probmdp$ with the probability variables $p_{s'}$.
We rewrite the NLP as a standard form of a \emph{quadratically-constrained quadratic program} (QCQP)~\cite{boyd_convex_optimization}.
Afterwards, we examine this QCQP in detail and show that it is nonconvex.  

In general, a QCQP is an optimization problem with a quadratic objective function and $m$ quadratic constraints, written as  
\begin{align}
 \text{minimize}   & \quad \vec{x}^T P_0 \vec{x} + q_0^T \vec{x}+r_0\label{eq:qcqp_min}\\
 \text{subject to} & \nonumber\\
\forall i \in \{1,\ldots,m\} &\quad   \vec{x}^T P_i \vec{x} + q_i^T \vec{x}+r_i \leq 0,\label{eq:qcqp_constraints}
\end{align}
where $\vec{x}$ is a vector of \emph{variables}, and the coefficients are $P_i \in \R^{n \times n}$, $g_i \in \R^n$, $r_i \in \R$ for $0\leq i\leq m$. 
We assume $P_0,\ldots, P_m$ are symmetric without loss of generality.
Constraints of the form $\vec{x}^T P_i \vec{x} + q_i^T \vec{x}+r_i = 0$ are encoded by
\begin{align*}
	 \vec{x}^T P_i \vec{x} + q_i^T \vec{x}+r_i \leq 0\text{ and }-\vec{x}^T P_i \vec{x} - q_i^T \vec{x} - r_i \leq 0\ .
\end{align*}
\paragraph{Properties of QCQPs.} We discuss properties of all matrices $P_i$ for $0\leq i\leq m$.
If all $P_i=0$, the function $q_i^T \vec{x}+r_i$ is \emph{affine}, and the QCQP is in fact an LP.
If every $P_i$ is PSD, the function $\vec{x}^T P_i \vec{x} + q_i^T \vec{x}+r_i$ is \emph{convex}, and the QCQP is a \emph{convex optimization problem}, that can be solved in polynomial time~\cite{alizadeh2003second}.
If any $P_i$ is not PSD, the resulting QCQP is nonconvex. The problem of finding a feasible solution in a nonconvex QCQP is NP-hard~\cite{burer2012milp}. 

To ease the presentation, we transform the quadratic constraints in the NLP in~\eqref{eq:min_mdp} -- \eqref{eq:well-defined_eps} to the standard QCQP form in~\eqref{eq:qcqp_min} -- \eqref{eq:qcqp_constraints}:
\begin{align}
			\text{minimize} &\quad p_{\sinit}\label{eq:min_qcqp}\\
			\text{subject to} & \quad \nonumber\\
			\forall s\in T.	 &\quad p_s=1\label{eq:targetprob_qcqp}\\
			\forall s,s'\in S. \forall\act\in\ActS(s).	 &\quad \probmdp(s,\act,s')\geq \epsgraph\label{eq:well-defined_probs_qcqp1}\\
			\forall s\in S.\forall \act\in\ActS(s).	 &\quad \sum_{s'\in S}\probmdp(s,\act,s')=1\label{eq:well-defined_probs_qcqp}\\
					&\quad \lambda \geq  p_{\sinit} \label{eq:probthreshold_qcqp}\\
			\forall s\in S\setminus T.	\forall \act\in \ActS(s). &\quad p_s \geq \vec{x}^{T}P_{s,\act} \vec{x}+q_{s,\act}^T\vec{x}, \label{eq:probcomputation_qcqp}
		\end{align}
where $\vec{x}$ is a vector consisting of the probability variables $p_s$ for all $s\in S$ and the pMDP parameters from $V$, \ie, $\vec{x}$ has $|S|+|V|$ rows. 
Furthermore, $P_{s,\act}\in\R^{(|S|+|V|)\times(|S|+|V|)}$ is a symmetric  matrix, and $q_{s,\act}\in\R^{(|S|+|V|)}$.

\paragraph{Construction of the QCQP.}
We use the matrix $P_{s,\act}$ to capture the \emph{quadratic} part and the vector $q_{s,\act}$ to capture the \emph{affine} part in~\eqref{eq:probcomputation_qcqp}.
More precisely, consider an affine function $\probmdp(s,\act,s')=a\cdot v+b$ with $a,b\in\R$.
The function occurs in the constraint~\eqref{eq:probcomputation_mdp} as part of the function $(a\cdot v+b)\cdot p_{s'}$.
The quadratic part thus occurs as $a\cdot v\cdot p_{s'}$ and the affine part as $b \cdot p_{s'}$.

We first consider 
the product $\vec{x}^TP_{s,\act}\vec{x}$, which denotes the sum over all products of entries in $\vec{x}$.
Thus, in $P_{s,\act}$, each row or column corresponds either to a probability variable $p_s$ for a state $s\in S$ or to a parameter $v \in V$. 
In fact, the cells indexed $(v, p_{s'})$ and $(p_{s'}, v)$ correspond to the product of these variables.
These two entries are summed in $\vec{x}^TP_{s,\act}\vec{x}$.
In $P_{s,\act}$, the sum is reflected by two entries $\nicefrac{a}{2}$ in the cells $(v, p_{s'})$ and $(p_{s'}, v)$.
Then $P_{s,\act}$ is a symmetric matrix, as required.
Similarly, we construct $q_{s,\act}$; the entry corresponding to $p_{s'}$ is set to $b$. 

We do not modify the affine functions in~\eqref{eq:targetprob_qcqp} -- \eqref{eq:probthreshold_qcqp} for the QCQP form.
\begin{example}\label{ex:qcqp_transformation}
Recall Example~\ref{ex:nlp}. 
We reformulate the NLP in~\eqref{eq:nlp-ex1} -- \eqref{eq:nlp-ex2} as a QCQP in the form of~\eqref{eq:min_qcqp} -- \eqref{eq:probcomputation_qcqp} using the same variables.

\begin{align*}
			\text{minimize} & \quad p_{s_0} \\
			\text{subject to}& \quad \nonumber\\
			&\quad p_{s_3}=1\\
			&\quad  \lambda \geq  p_{s_0} \geq \begin{bmatrix}
    v \\
    p_{s_1}
\end{bmatrix}^T P_{s_0} \begin{bmatrix}
    v \\
    p_{s_1}
\end{bmatrix}= \begin{bmatrix}
    v \\
    p_{s_1}
\end{bmatrix}^T \begin{bmatrix}
    0       & 0.5 \\
    0.5       & 0
\end{bmatrix}  \begin{bmatrix}
    v \\
    p_{s_1}
\end{bmatrix}\\
& \quad p_{s_1} \geq   \begin{bmatrix}
    v \\
    p_{s_2}
\end{bmatrix}^T P_{s_1} \begin{bmatrix}
    v \\
    p_{s_2}
\end{bmatrix}= \begin{bmatrix}
    v \\
    p_{s_2}
\end{bmatrix}^T \begin{bmatrix}
    0       & -0.5 \\
    -0.5       & 0
\end{bmatrix}  \begin{bmatrix}
    v \\
    p_{s_2}
\end{bmatrix} +p_{s_2}\\
&\quad   p_{s_2} \geq 
\begin{bmatrix}
    v \\
    p_{s_3}
\end{bmatrix}^T P_{s_2} \begin{bmatrix}
    v \\
    p_{s_3}
\end{bmatrix}= \begin{bmatrix}
    v \\
    p_{s_3}
\end{bmatrix}^T \begin{bmatrix}
    0       & 0.5 \\
    0.5       & 0
\end{bmatrix}  \begin{bmatrix}
    v \\
    p_{s_3}
\end{bmatrix} \\
			&\quad 1-\epsgraph \geq v \geq \epsgraph.
		\end{align*}
\end{example}
%

%
\begin{theorem}
	The QCQP in~\eqref{eq:min_qcqp} -- \eqref{eq:probthreshold_qcqp} is nonconvex in general.
\end{theorem}

\begin{proof}
The matrices $P_{s_0}, P_{s_1}, P_{s_2}$ in Example~\ref{ex:qcqp_transformation} have an eigenvalue of $-0.5$ and are not PSD.
Thus, the constraints and the resulting QCQP are nonconvex. \qed
\end{proof}





\section{Efficient pMDP Synthesis via Convexification}
\label{sec:ccp}
The QCQP in \eqref{eq:min_qcqp} -- \eqref{eq:probthreshold_qcqp} is nonconvex and hard to solve in general.
We provide a solution by employing a heuristic called the \emph{convex-concave procedure} (CCP)~\cite{lipp2016variations}, which relies on the ability to efficiently solve convex optimization problems.

The CCP computes a \emph{local optimum} of a non-convex \emph{difference-of-convex} (DC) problem. 
A DC problem has the form
	\begin{align}
			\text{minimize } &\quad f_0(\vec{x})-g_0(\vec{x})\\
			\text{subject to}& \nonumber \\
					\forall i=1,\ldots,m.		 &\quad f_i(\vec{x})-g_i(\vec{x})\leq 0,\label{eq:ccp_cons}
		\end{align}
where for $i=0,1,\ldots,m$, $f_i(\vec{x})\colon \R^n \rightarrow \R$ and $g_i(\vec{x})\colon \R^n \rightarrow \R$ are convex. 
The functions $-g_i(\vec{x})$ are \emph{concave}.
Every quadratic function can be rewritten as a DC function.
Consider the indefinite quadratic function $\vec{x}^{T}P_{s,\act} \vec{x}+q_{s,\act}^T\vec{x}$ from~\eqref{eq:probcomputation_qcqp}.
We decompose the matrix $P_{s,\act}$ into the difference of two matrices
\begin{align*}
	P_{s,\act} = P^+_{s,\act} - P^-_{s,\act}\text{ with } P^+_{s,\act}=P_{s,\act}+tI \text{ and } P^-_{s,\act}=tI, 
\end{align*}
where $I$ is the identity matrix, and $t \in \R_{+}$ is sufficiently large to render $P^+_{s,\act}$ PSD, e.g., larger than the largest eigenvalue of $P_{s,\act}$. 
Then, we rewrite $\vec{x}^{T}P_{s,\act} \vec{x}+q_{s,\act}^T\vec{x}$ as $\left(\vec{x}^{T}P^+_{s,\act}\vec{x}+q_{s,\act}^T\vec{x}\right) - \vec{x}^{T}P^-_{s,\act} \vec{x}$, which is in the form of~\eqref{eq:ccp_cons}.
\begin{example}
\label{ex:dc}
Recall the pMC in Fig.~\ref{fig:pmc_reform} and the QCQP from Example~\ref{ex:qcqp_transformation}.  All matrices $P_s$ of the QCQP are not PSD. We construct a DC problem with $t=1$ for all $P_s$:

\begin{align*}
			\text{minimize } &\quad p_{s_0} \\
			\text{subject to} & \nonumber\\
			&\quad p_{s_3}=1\\
			&\quad \lambda \geq p_{s_0} \geq \begin{bmatrix}
    v \\
    p_{s_1}
\end{bmatrix}^T \begin{bmatrix}
    1       & 0.5 \\
    0.5       & 1
\end{bmatrix} \begin{bmatrix}
    v \\
    p_{s_1}
\end{bmatrix} - \begin{bmatrix}
    v \\
    p_{s_1}
\end{bmatrix}^T \begin{bmatrix}
    1       & 0 \\
    0       & 1
\end{bmatrix} \begin{bmatrix}
    v \\
    p_{s_1}
\end{bmatrix} \\
&\quad p_{s_1} \geq \begin{bmatrix}
    v \\
    p_{s_2}
\end{bmatrix}^T \begin{bmatrix}
    1       & -0.5 \\
    -0.5       & 1
\end{bmatrix} \begin{bmatrix}
    v \\
    p_{s_2}
\end{bmatrix} - \begin{bmatrix}
    v \\
    p_{s_2}
\end{bmatrix}^T \begin{bmatrix}
    1       & 0 \\
    0       & 1
\end{bmatrix} \begin{bmatrix}
    v \\
    p_{s_2}
\end{bmatrix} +p_{s_2}\\
&\quad p_{s_2} \geq  \begin{bmatrix}
    v \\
    p_{s_3}
\end{bmatrix}^T \begin{bmatrix}
    1       & 0.5 \\
    0.5       & 1
\end{bmatrix} \begin{bmatrix}
    v \\
    p_{s_3}
\end{bmatrix} - \begin{bmatrix}
    v \\
    p_{s_3}
\end{bmatrix}^T \begin{bmatrix}
    1       & 0 \\
    0       & 1
\end{bmatrix} \begin{bmatrix}
    v \\
    p_{s_3}
\end{bmatrix} \\
			&\quad 1-\epsgraph  \geq v \geq \epsgraph.
		\end{align*}
We have $\vec{x}=(p_{s_1},p_{s_2},p_{s_3},v)$ and an initial assignment $\hat x=(\hat p_{s_1},\hat p_{s_2},\hat p_{s_3},\hat v)$. 
\end{example}

\paragraph{CCP approach.}
For the resulting DC problem, we consider the iterative \emph{penalty CCP method}~\cite{lipp2016variations}.
The procedure is initialized with any initial assignment $\hat{x}$ of the variables $\vec{x}$.
In the \emph{convexification} stage, we compute affine approximations in form of a linearization of $g_{i}(\vec{x})$ around $\hat{x}$:
	\begin{align*}
	\bar{g}_i(\vec{x})=g_i(\hat{x})+\nabla g_i(\hat{x})^{T}(\vec{x}-\hat{x}),
		\end{align*}	
where $\nabla g_i$ is the gradient of the functions $g_i(\vec{x})$ at $\hat{x}$. 
Then, we replace the DC function $f_i(\vec{x})-g_i(\vec{x})$ by  $f_i(\vec{x})-\bar{g}_i(\vec{x})$, which is a \emph{convex over-approximation} of the original function.
A feasible assignment for the resulting over-approximated and \emph{convex} DC problem is also feasible for the original DC problem.
 
To find such a feasible assignment, a \emph{penalty variable} $k_{s,\act}$ for all $s\in S \setminus T$ and $\act\in\Act$ is added to all convexified constraints.
Solving the resulting problem then seeks to minimize the violation of the original DC constraints by minimizing the sum of the penalty variables.
The resulting convex problem is written as
\begin{align}
			\text{minimize} &\quad p_{\sinit}+ \tau\sum_{\forall s\in S\setminus T}\sum_{\forall \act \in \Act}k_{s,\act}\label{eq:min_qcqp_ccp}\\
			\text{subject to}\nonumber\\
			\forall s\in T.	 &\quad p_s=1\label{eq:targetprob_ccp}\\
			\forall s,s'\in S. \forall\act\in\ActS(s).	 &\quad \probmdp(s,\act,s')\geq \epsgraph \label{eq:well-defined_probs_ccp}\\
			\forall s\in S.\forall \act \in \ActS(s).	 &\quad \sum_{s'\in S}\probmdp(s,\act,s')=1\label{eq:well-defined_probs_ccp1}\\
								 &\quad \lambda \geq p_{\sinit}
								 \label{eq:probthreshold_ccp}\\
								 			\forall s\in S\setminus T. \forall \act \in \ActS(s)	&\quad k_{s,\act} +p_s \geq \vec{x}^{T}P^+_{s,\act} \vec{x}+q_{s,\act}^T\vec{x}-\hat{x}^{T}P^-_{s,\act}(2\vec{x}-\hat{x})\label{eq:probcomputation_ccp}\\
			\forall s\in S\setminus T.	\forall \act \in \ActS(s) &\quad k_{s,\act} \geq 0, \label{eq:slackvariable_ccp}
		\end{align}
where $\tau > 0$ is a fixed \emph{penalty parameter}, and the gradient of $\vec{x}^TP^-_{s,\act}\vec{x}$ is $2 \cdot P^-_{s,\act}\hat{x}$.  
This convexified DC problem is in fact a convex QCQP. 
The changed objective now makes the constraint \eqref{eq:probthreshold_ccp} important.

\begin{example}
Recall the pMC in Fig.~\ref{fig:pmc_reform} and the DC problem from Example~\ref{ex:dc}. 
We introduce the penalty variables $k_{s_i}$ and assume a fixed $\tau$.
We linearize around $\hat x$. The resulting convex problem is:

 \begin{align*}
		&	\text{minimize } \quad p_{s_0}+\tau\sum_{i=0}^2 k_{s_i} \\
		&	\text{subject to}  \nonumber\\
			&\quad p_{s_3}=1\\
			&\quad \lambda \geq p_{s_0} \\
			&\quad  k_{s_0}+p_{s_0} \geq \begin{bmatrix}
    v \\
    p_{s_1}
\end{bmatrix}^T \begin{bmatrix}
    1       & 0.5 \\
    0.5       & 1
\end{bmatrix} \begin{bmatrix}
    v \\
    p_{s_1}
\end{bmatrix} - \begin{bmatrix}
    \hat{v} \\
    \hat{p}_{s_1}
\end{bmatrix}^T \begin{bmatrix}
    1       & 0 \\
    0       & 1
\end{bmatrix} \begin{bmatrix}
    2\cdot v-\hat{v} \\
    2\cdot p_{s_1}-\hat{p}_{s_1}
\end{bmatrix} \\
&\quad  k_{s_1}+p_{s_1} \geq \begin{bmatrix}
    v \\
    p_{s_1}
\end{bmatrix}^T \begin{bmatrix}
    1       & -0.5 \\
    -0.5       & 1
\end{bmatrix} \begin{bmatrix}
    v \\
    p_{s_2}
\end{bmatrix} - \begin{bmatrix}
    \hat{v} \\
    \hat{p}_{s_2}
\end{bmatrix}^T \begin{bmatrix}
    1       & 0 \\
    0       & 1
\end{bmatrix} \begin{bmatrix}
    2\cdot v-\hat{v} \\
    2\cdot p_{s_2}-\hat{p}_{s_2}
\end{bmatrix} +p_{s_2} \\
&\quad  k_{s_2}+p_{s_2} \geq \begin{bmatrix}
    v \\
    p_{s_3}
\end{bmatrix}^T \begin{bmatrix}
    1       & 0.5 \\
    0.5       & 1
\end{bmatrix} \begin{bmatrix}
    v \\
    p_{s_3}
\end{bmatrix} - \begin{bmatrix}
    \hat{v} \\
    \hat{p}_{s_3}
\end{bmatrix}^T \begin{bmatrix}
    1       & 0 \\
    0       & 1
\end{bmatrix} \begin{bmatrix}
   2\cdot v- \hat{v} \\
   2 \cdot p_{s_3} - \hat{p}_{s_3}
\end{bmatrix}  \\
			&\quad  1-\epsgraph \geq v\geq \epsgraph \\
& \quad k_{s_0}\geq 0, k_{s_1}\geq 0, k_{s_2} \geq 0. 
		\end{align*}
\end{example}
If all penalty variables are assigned to zero, we can terminate the algorithm immediately, for the proof \iftoggle{TR}{see Appendix~\ref{sec:Proof_thm}.}{cf.~\cite{TR}.}
\begin{theorem}
	\label{thm:main}
	A satisfying assignment of the convex DC problem in~\eqref{eq:min_qcqp_ccp} -- \eqref{eq:slackvariable_ccp}
	\begin{align*}
		\text{with}\qquad  \tau\sum_{\forall s\in S\setminus T}\sum_{\forall \act \in \Act}k_{s,\act}=0		
	\end{align*}
 is a feasible solution to Problem 1.
\end{theorem}
%
If any of the penalty variables are assigned to a positive value, we update the penalty parameter $\tau$ by $\mu + \tau$ for a $\mu > 0$, until an upper limit for $\tau_{\mathit{max}}$ is reached to avoid numerical problems.
Then, we compute a linearization of the $g_i$ functions around the current (not feasible) solution and solve the resulting problem. 
This procedure is repeated until we find a feasible solution. 
If the procedure converges to an infeasible solution, it may be restarted with an adapted initial $\hat{x}$.


\subsection*{Efficiency Improvements in the Convex-Concave Procedure}\label{sec:efficiencyimprovements}

\paragraph{Better convexification.}
We can use the previous transformation 
 to perform CCP, but it involves expensive matrix operations, including computing the numerous eigenvalues.
Observe that 
the matrices $P_{s,\act}$ and vectors $q_{s,\act}$ are sparse. 
Then, the eigenvalue method introduces more occurrences of the  variables in every constraint,
 and thereby increases the approximation error during convexification.

We use an alternative convexification:
Consider the bilinear function $h = 2c\cdot yz$, where $y$ and $z$ are variables, and $c \in \R_+$.
We rewrite $h$ equivalently to $h + c(y^2+z^2) - c(y^2+z^2)$.
Then, we rewrite $h + c(y^2+z^2)$ as $c(y+z)^2$.
We obtain $h = c(y+z)^2 -c(y^2+z^2)$.
The function $c(y+z)^2$ is a quadratic convex function, and we convexify the function $-c(y^2+z^2)$  as $-c(\hat{y}^2+\hat{z}^2)+2c(\hat{y}^2+\hat{z}^2-y \hat{y} - z \hat{z})$, where $\hat{y}$ and $\hat{z}$ are the assignments as before.
We convexify the bilinear function  $h = 2c\cdot yz$ with $c \in \R_-$ analogously.
Consequently, we reduce the occurrences of variables for sparse matrices compared to the eigenvalue method.

\paragraph{Integrating model checking with CCP.}
In each iteration of the CCP, we obtain values $\hat{v}$ which give rise to a parameter instantiation.
Model checking at these instantiations is a good heuristic to allow for \emph{early termination}. 
We check whether the values $\hat{v}$ already induce a feasible solution to the original NLP, even though the penalty variables have not converged to zero.

Additionally, instead of instantiating the initial probability values $\hat{p}_s$ in iteration $i+1$, we may use the model checking result of the MDP instantiated at $\hat{v}$ from iteration $i$. 
Model checking ensures that the probability variables are consistent with the parameter variables, i.e., that the constraints describing the transition relation in the original NLP are all met.
Using the model checking results overcomes problems with local optima.
Small violations in \eqref{eq:probcomputation_ccp}, \ie, small $k_{s,\act}$ values can lead to big differences in the actual probability valuations.
 Then, the CCP may be trapped in poor local optima, where the sum of constraint violations is small, but the violation for the probability threshold is too large.

\paragraph{Algorithmic improvements.} We list three key improvements that we make as opposed to a naive implementation of the approaches.
\begin{inparaenum}[(1)]
\item 	
We efficiently precompute the states $s \in S$ that reach target states with probability $0$ or $1$. Then, we simplify the NLP in~\eqref{eq:min_mdp} -- \eqref{eq:probcomputation_mdp} accordingly.
\item 
Often, all instantiations with admissible parameter values yield well-defined MDPs. We verify this property via an easy preprocessing.
Then, we omit the constraints for the well-definedness.
\item 
Parts of the encoding are untouched over multiple CCP iterations.
Instead of rebuilding the encoding, we only update constraints which contain iteration-dependent values. 
The update is based on a preprocessed representation of the model.
The improvement is two-fold: We spend less time constructing the encoding, and the solver reuses previous results, making solving up to three times faster.
\end{inparaenum}



\section{Experiments}\label{sec:experiments}
\subsection{Implementation}
We implement the CCP with the discussed efficiency improvements from Sect.~\ref{sec:efficiencyimprovements} in the parameter synthesis framework \prophesy~\cite{dehnert-et-al-cav-2015}.
We use the probabilistic model checker \storm~\cite{DBLP:conf/cav/DehnertJK017} to extract an explicit representation of  an pMDP.
We keep the pMDP in memory, and update the parameter instantiations using an efficient data structure to enable efficient repetitive model checking.
To solve convex QCQPs, we use Gurobi 7.5~\cite{gurobi}, configured for numerical stability.

\paragraph{Tuning constants.}
Optimally, we would initialize the CCP procedure, \ie, $\hat{v}$ (for the parameters) and $\hat{p}_s$ (for the probability variables), with a feasible point, but that would require to already solve Problem~\ref{prob:pmdpsyn}. 
Instead, we instantiate $\hat{v}$ as the center of the parameter space, and thereby minimize the worst-case distance to a feasible solution.
For $\hat{p}_s$, we use the threshold $\lambda$ from the specification $\reachProplT$ to initialize the probability variables, and analogously for expected cost.
We initialize the penalty parameter $\tau=0.05$ for reachability, and $\tau=5$ for expected cost, a conservative number in the same order of magnitude as the values $\hat{p}_s$.
As expected cost evaluations have wider ranges than probability evaluations, a larger $\tau$ is sensible.
We pick $\mu = \max_{s\in S\setminus T }\hat{p}_s$.
We update $\tau$ by adding $\mu$ after each iteration. 
Empirically, increasing $\tau$ with bigger steps is beneficial for the run time, but induces more numerical instability.
In contrast, in the literature, the update parameter $\mu$ is frequently used as a constant, \ie , it is not updated between the iterations. 
In, e.g,~\cite{lipp2016variations}, $\tau$ is multiplied by $\mu$ after each iteration.

\subsection{Evaluation}
\paragraph{Set-up.}
We evaluate on a HP BL685C G7 with 48 2 GHz cores, a 32 GB memory limit, and 1800 seconds time limit; the implementation only using a single thread.
The task is to find feasible parameter valuation for pMCs and pMDPs with non-trivial upper/lower thresholds on probabilities/costs in the specifications, as in Problem~\ref{prob:pmdpsyn}. 
We ask for a well-defined valuation of the parameters, with $\epsgraph=10^{-5}$.
We run all the approaches with the exact same configuration of \storm. For pMCs, we enable weak bisimulation, which is beneficial for all presented examples. 
We do not use bisimulation for pMDPs.

We compare runtimes with a particle-swarm optimization (PSO) and two SMT-based approaches.
PSO is a heuristic sampling approach which searches the parameter space, inspired by~\cite{chen2013model}.
For each valuation, PSO performs model checking \emph{without} rebuilding the model, rather it adapts the matrix from previous valuations.
As PSO is a randomized procedure, we run it with random seeds 0--19. 
The PSO implementation requires the well-defined parameter regions to constitute a hyper-rectangle, as proper sampling from polygons is a non-trivial task.
The first SMT approach directly solves the NLP \eqref{eq:targetprob_mdp} -- \eqref{eq:well-defined_eps} using the SMT solver Z3~\cite{demoura_nlsat}.
The second SMT approach preprocesses the NLP using state elimination~\cite{Daw04} as implemented in, e.g., \param, \prism and \storm. 

We additionally compare against a prototype of the geometric programming (GP) approach~\cite{DBLP:conf/tacas/Cubuktepe0JKPPT17} based on CvxPy~\cite{cvxpy} and the solver SCS~\cite{scs}, and the QCQP package~\cite{park2017general}, which implements several heuristics, including a naive CCP approach, for nonconvex QCQPs. 
Due to numerical instabilities, we could not automatically apply these two approaches to a wide range of benchmarks.

\paragraph{Benchmarks.}
We include the standard pMC benchmarks from the \param website, which contain two parameters.
We furthermore have a rich selection of strategy synthesis problems obtained from partially observable MDPs (POMDPs), cf.~\cite{DBLP:journals/corr/abs-1710-10294}:
GridX are gridworld problems with trap states (A), finite horizons  (B), or movement costs (C). 
Maze is a navigation problem.
Network and Repudiation originate from distributed protocols.
We obtain the pMDP benchmarks either from the \param website, or as parametric variants to existing \prism case studies, and describe randomized distributed protocols.

\paragraph{Results.}
Table~\ref{tab:results_pmc} contains an overview of the results.
The first two columns refer to the benchmark instance, the next column to the specification evaluated. 
We give the states (States), transitions (Trans.) and parameters (Par.) \emph{in the bisimulation quotient}, which is then used for further evaluation.
We then give the \emph{minimum} (tmin), the \emph{maximum} (tmax) and \emph{average} (\textbf{tavg}) runtime (in seconds) for PSO with different seeds, the best runtime obtained using SMT (\textbf{t}), and the runtime for CCP (\textbf{t}). 
For CCP, we additionally give the fraction (in percent) of time spent in Gurobi (solv), and the number of CCP iterations (iter).
Table~\ref{tab:results_pmdp} additionally contains the number of actions (Act) for pMDPs.
The boldfaced measures \textbf{tavg}, and \textbf{t} for both SMT and CCP are the important measures to compare. 
Boldface values are the ones with the best performance for a specific benchmark.

\setlength{\tabcolsep}{2pt} 

\begin{table}[t]
\centering
\footnotesize{
\caption{pMC benchmark results}
\label{tab:results_pmc}
\scalebox{0.9}{
\begin{tabular}{lll|rrr|rrr|r|rrr}
\multicolumn{3}{c|}{Problem} & \multicolumn{3}{c|}{Info} & \multicolumn{3}{c|}{PSO} & SMT & \multicolumn{3}{c}{CCP} \\
Set    & Inst    & Spec   & States   & Trans.  & Par.      & tmin   & tmax   & \textbf{tavg}  & \textbf{t}   & \textbf{t}    & solv    & iter    \\\hline
Brp  & 16,2         & $\p_{\leq 0.1}$       & 98    & 194  & \highlight{2}     &  0     & 0       &  \textbf{0}     & 40    & 0     &  30\%  &   3 \\
Brp  & 512,5         & $\p_{\leq 0.1}$       & 6146    & 12290  & \highlight{2}     &  24     & 36       &  \textbf{28}     & \TO    & 33     &  24\%  &   3 \\
Crowds  & 10,5         & $\p_{\leq 0.1}$       & 42    & 82  & \highlight{2}     &  4     & 5       &  5     & 8    & \textbf{4}     &  2\%  &   4 \\
Nand  & 5,10         & $\p_{\leq 0.05}$       & 10492    & 20982  & \highlight{2}     &  21     & 51       &  28     & \TO    & \textbf{22}     &  21\%  &   2 \\
Zeroconf  & 10000         & $\EV_{\leq 10010}$       & 10003    & 20004  & \highlight{2}     &  2      & 4       &  \textbf{3}     & \TO    & 57     &  81\%  & 3  \\\hline
GridA  & 4         & $\p_{\geq 0.84}$       & 1026    & 2098  & \highlight{72}     &   11     &   11     &  \textbf{11}     &  \TO    & 22     &  81\%  & 11  \\
GridB  & 8,5         & $\p_{\geq 0.84}$       & 8653    & 17369  & \highlight{700}     &    409    &   440     & 427      & \TO    & \textbf{213}     &  84\%  & 8  \\
GridB  & 10,6         & $\p_{\geq 0.84}$       & 16941    & 33958  & \highlight{1290}     &  533      &   567     &  553     &  \TO   & \textbf{426}     &  84\%  & 7  \\
GridC  & 6         & $\EV_{\leq 4.8}$       & 1665    & 305  & \highlight{168}     &  261      &  274      &  267     &  \TO   & \textbf{169}     &  90\%  & 23  \\
Maze  & 5         & $\EV_{\leq 14}$       & 1303    & 2658  & \highlight{590}     &  213  & 230 & 219         &  \TO   &   \textbf{67}   &  89\%  & 8  \\
Maze  & 5         & $\EV_{\leq 6}$       & 1303    & 2658  & \highlight{590}     &     -- &  -- & {\TO}          &  \TO   & \textbf{422}     &  85\%  & 97  \\

Maze  & 7         & $\EV_{\leq 6} $       & 2580    & 5233  & \highlight{1176}     &    -- &  -- & {\TO}      & \TO    & \textbf{740}     &  90\%  & 60  \\
Netw  & 5,2         & $\EV_{\leq11.5}$       & 21746    & 63158  & \highlight{2420}     &  312      & 523       &   359    &  \TO   & \textbf{207}     & 39\%   & 3  \\

Netw  & 5,2         & $\EV_{\leq10.5}$       & 21746    & 63158  & \highlight{2420}     &        -- &  -- & {\TO}        &  \TO   & \textbf{210}     & 38\%   & 4  \\
Netw  & 4,3         & $\EV_{\leq11.5}$       & 38055    & 97335  & \highlight{4545}     &        -- &  -- & {\TO}        &  \TO   & \MO     & -   & -  \\

Repud  & 8,5         &   $\p_{\geq 0.1}$      & 1487    & 3002  & \highlight{360}     &   16     & 22       &   18    & \TO    & \textbf{4}     & 36\%   & 2  \\
Repud  & 8,5         & $\p_{\leq 0.05}$        & 1487    & 3002  & \highlight{360}     &    273    &   324     &   293    &  \TO   & \textbf{14}     &  72\%  &  4 \\
Repud  & 16,2         &  $\p_{\leq 0.01}$      & 790    & 1606  & \highlight{96}     &       -- &  -- & {\TO}       & \TO    &  \textbf{15}    & 78\%   & 9  \\
Repud  & 16,2         &  $\p_{\geq 0.062}$      & 790    & 1606  & \highlight{96}     &       -- &  -- & {\TO}       & \TO    &  \TO    & -   & -  \\\hline

\end{tabular}
}}
\end{table}
\begin{table}[t]
\centering
\caption{pMDP benchmark results}
\label{tab:results_pmdp}
\footnotesize{
\scalebox{0.9}{
\begin{tabular}{lll|rrrr|rrr|r|rrr}
\multicolumn{3}{c|}{Problem} & \multicolumn{4}{c|}{Info}   & \multicolumn{3}{c|}{PSO} & SMT & \multicolumn{3}{c}{CCP} \\
Set    & Inst    & Spec   & States & Act & Trans. & Par. & tmin   & tmax   & \textbf{tavg}  & \textbf{t}   & \textbf{t}    & solv    & iter    \\\hline

BRP & 4,128 & $\p_{\leq 0.1}$ &17131 & 17396 & 23094 & \highlight{2}  & 45 & 47 & 46 & \TO & \textbf{39}
 & 33\%  & 4 \\  
Coin & 32 &$\EV_{\leq 500}$ & 4112 & 6160 & 7692 & \highlight{2} & 117 & 119 & \textbf{118} & \TO & \TO & - & - \\ 
CoinX & 32 & $\EV_{\leq 210}$ & 16448 & 24640 & 30768 & \highlight{2} & 1196 & 1222 & 1208 & \TO & \textbf{32} & 78\% & 3 \\ 
Zeroconf       & 1        & $\p_{\geq 0.99}$ & 31402      &  55678      &  70643   &    \highlight{3}   & 18    &  19      & \textbf{19}        &  \TO      &     79     &    82\%     &   2 \\
CSMA & 2,4 & $\EV_{\leq 69.3}$ & 7958 & 7988 & 10594 & \highlight{26} & n.s. & n.s. & n.s. & \TO & \textbf{79}  & 86\% & 10  \\
Virus & - & $\EV_{\leq 10}$ & 809 & 3371 & 6741 & \highlight{18} &  113 & 113 & 113 & \TO & \textbf{13} & 76\% & 4 \\
Wlan & 0 & $\EV_{\leq 580}$ & 2954 & 3972  & 5202 & \highlight{15} & n.s. & n.s. & n.s. & \TO & \textbf{7} & 72\% & 2 \\\hline

\end{tabular}
}}
\end{table}

There is a constant overhead for model building, which is in particular large if the bisimulation quotient computation is expensive, see the small fraction of time spent solving CCPs for Crowds.
For the more challenging models, this overhead is negligible. Roughly 80--90\% of the time is spent within \tool{Gurobi} in these models, the remainder is used to feed the CCPs into Gurobi.
A specification threshold closer to the (global) optimum typically induces a higher number of iterations (see Maze or Netw with different threshold).
For the pMDP Coin, optimal parameter values are on the boundary of the parameter space and quickly reached by PSO. The small parameter values together with the rather large expected costs are numerically challenging for CCP. 
 For CoinX, the parameter values are in the interior of the parameter space and harder to hit via sampling. 
 For CCP, the difference between small and large coefficients is smaller than in Coin, which yields better convergence behavior.
 The benchmarks CSMA and WLAN are currently not supported by PSO due to the non-rectangular well-defined parameter space.
 
 CCP does not solve all instances: 
In Netw (4,3), CCP exceeds the memory limit.
In Repud, finding values close the global optimum requires too much time. 
While the thresholds used here are close to the global optima, actually finding the global optimum itself is always challenging.

\paragraph{Effect of integrating model checking for CCP.} 

The benchmark-set Maze profits most: Discarding the model checking results in our CCP implementation always yields time-outs, even for the rather simple Maze, with threshold 14, which is solved with usage of model checking results within 30 seconds.
Here, using model checking results thus yields a speed-up by a factor of at least 60.
More typical examples are Netw, where discarding the model checking results yields a factor 5 performance penalty.
The Repud examples do not significantly profit from using intermediate model checking results.

\paragraph{Evaluation of the QCQP package, GP and SMT.}
We evaluate the GP on pMCs with two parameters: 
For the smaller BRP instance, the procedure takes 90 seconds, for Crowds 14 seconds. Other instances yield timeouts. 
We also evaluate the QCQP package on some pMCs. 
For the smaller BRP instance, the package finds a feasible solution after 113 seconds. 
For the Crowds instance, it takes 13 seconds. 
For a Repud instance with 44 states, 
 and 26 parameters, the package takes 54 seconds and returns a solution that violates the specification. CCP with integrated model checking takes less than a second.

The results in Tables~\ref{tab:results_pmc} and~\ref{tab:results_pmdp} make obvious that \emph{SMT is not competitive}, irrespectively whether the NLP is preprocessed via state elimination.
Moreover, state elimination (for pMCs) within the given time limit \emph{is only possible for those (considered) models with 2 parameters}, using either \prism, \param, or \storm.

\subsection{Discussion}

\paragraph{A tuned variant of CCP improves the state-of-the-art.}
Just applying out-of-the-box heuristics for QCQPs---like realized in the QCQP package or using our CCP implementation without integrated model checking---does not yield a scalable method.
To solve the nonconvex QCQP, we require a CCP with a clever encoding, cf.\ Sect.~\ref{sec:efficiencyimprovements}, and several algorithmic improvements.
State space reductions shrink the encoding, and model checking after each CCP iteration to terminate earlier typically saves 20\% of iterations. 
Especially when convergence is slow, model checking saves significantly more iterations.
Moreover, feeding model checking results into the CCP improves runtime by up to an additional order of magnitude, at negligible costs. 
These combined improvements to the CCP method outnumbers any solver-based approach by orders of magnitude, and is often superior to sampling based approaches, especially in the presence of many parameters.
Benchmarks with many parameters pose two challenges for sampling based approaches: 
Sampling is necessarily sparse due to the high dimension, and
 optimal parameter valuations for one parameter often depend significantly on other parameter values.


%

\paragraph{CCP performance can be boosted with particular benchmarks in mind.}
For most benchmarks, choosing larger values for $\tau$ and $\mu$ improves the performance. 
Furthermore, for particular benchmarks, we can find a better initial value for $\hat{p}_s$ and $\hat{v}$. 
These adaptions, however, are not suitable for a general framework. 
Values used here reflect a more balanced performance over several types of benchmarks.
On the downside, the dependency on the constants means that minor changes in the encoding may have significant, but hard to predict, effect.
For SMT-solvers, additional and superfluous constraints often help steering the solver, but in the context of CCP, it diminishes the performance.

\paragraph{Some benchmarks constitute numerically challenging problems.}
For specification thresholds close to global optima and for some expected cost specifications in general, feasible parameter values may be very small. 
Such extremal parameter values induce CCPs with large differences between the smallest and largest coefficient in the encoding, which are numerically challenging.
The pMDP benchmarks are more susceptible to such numerical issues.

\section{Conclusion and Future Work}\label{sec:conclusion}
We presented a new approach to parameter synthesis for pMDPs. 
To solve the underlying nonconvex optimization problem efficiently, we devised a method to efficiently employ a heuristic procedure with integrated model checking. 
The experiments showed that our method significantly improves the state-of-the-art.

In the future, we will investigate how to automatically handle nonaffine transition functions.
To further improve the performance, we will implement a hybrid approach between PSO and the CCP-based method.

\bibliographystyle{splncs04}
\bibliography{literature}
\iftoggle{TR}{
\newpage\appendix

\section{Proof of Theorem~\ref{thm:main}}
\begin{proof}
\label{sec:Proof_thm}

Since the assignment of the convexified DC problem in~\eqref{eq:min_qcqp_ccp} -- \eqref{eq:slackvariable_ccp} is feasible with 
	\begin{align*}
		\tau\sum_{\forall s\in S\setminus T}\sum_{\forall \act \in \Act}k_{s,\act}=0,		
	\end{align*}
we know that the assignment is feasible for the QCQP in~\eqref{eq:min_qcqp} -- \eqref{eq:probthreshold_qcqp} and, by definition, for the NLP in~\eqref{eq:min_mdp} -- \eqref{eq:well-defined_eps}. We will show that for every $s \in S$ we have $\pr(\mdp[u],\finally T,s)\leq p_s$, for any feasible assignment for the NLP in~\eqref{eq:min_mdp} -- \eqref{eq:well-defined_eps}.

For $s \in S$,  define $q_s = \pr_s(\mdp[u],\finally T)$ (the probability to reach $T$ from $s$ in $\mdp[u]$) and $x_s = q_s - p_s$. Let $S_{<}  = \{s \in S \mid p_s < q_s\}$.

For states $s \in T$ we have, by~\eqref{eq:targetprob_mdp} that $p_s=1=q_s = 1$, meaning that $s \not \in S_<$. For states $s$ with $q_s = 0$, \ie, states from which $T$ is almost surely not reachable, we have trivially $p _s \geq q_s$, also implying $s \not \in S_<$. Therefore,
for every $s \in S_<$, $p_s$ satisfies~\eqref{eq:probcomputation_mdp} and $T$ is reachable with positive probability.

Assume, for the sake of contradiction, that $S_{<} \neq \emptyset$, and let $x_{max}  =  \max\{x_s \mid s \in S\},$ and  $S_{max}  = \{s \in S \mid x_s  = x_{max}\}.$

The assumption that $S_< \neq \emptyset$ implies $x_{max} > 0$. Let $s \in S$ be such that $x_s = x_{max}$. Therefore, $s \in S_<$, and thus for all $\act \in \Act$, we have that 

 \begin{align}
 x_s \geq  \displaystyle \sum_{s'\in S}	\probmdp(s,\act,s')\cdot p_{s'}.
 \end{align}

On the other hand, there exists an $\act \in \Act$ such that we have
 \begin{align}
\displaystyle q_s =\displaystyle \sum_{s'\in S} \probmdp(s,\act,s')\cdot q_{s'}.
 \end{align}

Thus, 

 \begin{align}
\displaystyle q_s-p_s & \leq \displaystyle \sum_{s'\in S} \probmdp(s,\act,s')\cdot (q_{s'}-p_{s'}), \label{eq:proofineq}
 \end{align}
 
which is equivalent to

 \begin{align}
\displaystyle x_s  \leq  \displaystyle \sum_{s'\in S} \probmdp(s,\act,s')\cdot x_{s'}.
 \end{align}
Since for all $\act \in \Act$, and $s' \in S$, we have that $\probmdp(s,\act,s') \geq 0$ and $\displaystyle \sum_{s'\in S}\cdot  \probmdp(s,\act,s') = 1$, using the inequality in~\eqref{eq:proofineq} we establish
 \begin{align}
x_{max} = x_s & \leq  \displaystyle \sum_{s'\in S} \probmdp(s,\act,s')\cdot x_{s'}\\
&\leq \displaystyle \sum_{s'\in S} \probmdp(s,\act,s')\cdot x_{max}\\
&\leq x_{max} \displaystyle \sum_{s'\in S} \probmdp(s,\act,s')= x_{max}.
 \end{align}

Which implies that all the inequalities are equalities, meaning that
 \begin{align}
x_{max} = x_s & =  \displaystyle \sum_{s'\in S} \probmdp(s,\act,s')\cdot x_{s'}\\
&= x_{max} \cdot \displaystyle \sum_{s'\in S} \probmdp(s,\act,s').\label{eq:proofeq}
 \end{align}
 
The equation in~\eqref{eq:proofeq} gives us that $x_{s'} = x_{max} > 0$ for every successor $s'$ of $s$ in $\mdp[u]$. Since $s \in S_{max}$ was chosen arbitrarily, for every state $s \in S_{max}$, all successors of $s$ are also in $S_{max}$. As we established that $S_< \cap T = \emptyset$, it is necessary that $T$ is not reachable with positive probability from any $s \in S_{max}$, which is a contradiction with the fact that from every state in $S_<$, the set $T$ is reachable.\qed

\end{proof}

}

\end{document}